\documentclass[preprint,12pt,authoryear]{elsarticle}

\usepackage{graphicx}%
\usepackage{multirow}%
\usepackage{amsmath,amssymb,amsfonts}%
\usepackage{amsthm}%
\usepackage{mathrsfs}%
\usepackage[title]{appendix}%
\usepackage{xcolor}%
\usepackage{textcomp}%
\usepackage{manyfoot}%
\newtheorem{theorem}{Theorem}
\usepackage{booktabs}%
\usepackage{algorithm}%
\usepackage{algorithmicx}%
\usepackage{algpseudocode}%
\usepackage{listings}%
\usepackage{multirow}
\usepackage{caption}
\usepackage{hyperref}
\usepackage[font=normalsize,labelfont=sf,textfont=sf]{subcaption}
\usepackage{comment}
\usepackage{wrapfig}

\journal{Nuclear Physics B}

\begin{document}

\begin{frontmatter}



\title{Pruning and Malicious Injection: A Retraining-Free Backdoor Attack on Transformer Models} 


\author[1]{Taibiao Zhao}
\ead{tzhao3@lsu.edu}
\author[1]{Mingxuan Sun\corref{cor1}}
\ead{msun11@lsu.edu}
\author[3]{Hao Wang}
\ead{hwang9@stevens.edu}
\author[2]{Xiaobing Chen}
\ead{xchen87@lsu.edu}
\author[2]{Xiangwei Zhou}
\ead{xwzhou@lsu.edu}
\affiliation[1]{organization={Division of Computer Science and Engineering, Louisiana State University},
            addressline={Baton Rouge}, 
            city={Baton Rouge},
            postcode={70803}, 
            state={Louisiana},
            country={USA}}
\affiliation[2]{organization={Division of Electrical and Computer Engineering, Louisiana State University},
            addressline={Baton Rouge}, 
            city={Baton Rouge},
            postcode={70803}, 
            state={Louisiana},
            country={USA}}
\affiliation[3]{organization={School of Engineering and Science, Stevens Institute of Technology},
            addressline={Castle Point Terrace}, 
            city={Hoboken},
            postcode={07030}, 
            state={New Jersey},
            country={USA}}
\cortext[cor1]{Corresponding author}
\begin{abstract}
Transformer models have demonstrated exceptional performance and have become indispensable in computer vision (CV) and natural language processing (NLP) tasks. However, recent studies reveal that transformers are susceptible to backdoor attacks. Prior backdoor attack methods typically rely on retraining with clean data or altering the model architecture, both of which can be resource-intensive and intrusive. In this paper, we propose Head-wise Pruning and Malicious Injection (HPMI), a novel retraining-free backdoor attack on transformers that does not alter the model's architecture. Our approach requires only a small subset of the original data and basic knowledge of the model architecture, eliminating the need for retraining the target transformer. Technically, HPMI works by pruning the least important head and injecting a pre-trained malicious head to establish the backdoor. We provide a rigorous theoretical justification demonstrating that the implanted backdoor resists detection and removal by state-of-the-art defense techniques, under reasonable assumptions. Experimental evaluations across multiple datasets further validate the effectiveness of HPMI, showing that it 1) incurs negligible clean accuracy loss, 2) achieves at least 99.55\% attack success rate, and 3) bypasses four advanced defense mechanisms. Additionally, relative to state-of-the-art retraining-dependent attacks, HPMI achieves greater concealment and robustness against diverse defense strategies, while maintaining minimal impact on clean accuracy.
\end{abstract}




\end{frontmatter}


\section{Introduction}
Transformer networks~\citep{vaswani2017attention} have become state-of-the-art in machine learning, with variations like BERT~\citep{devlin2018bert}, GPT~\citep{gpt,gpt2,achiam2023gpt4}, ViT~\citep{dosovitskiy2020image}, Swin Transformer~\citep{liu2021swin}, and DeiT~\citep{touvron2021training} widely used in natural language processing (NLP) and computer vision (CV) tasks. 
The excellent performance, however, demands an ever rising amount of computing power increasing exponentially with the model size and thus escalating energy costs and carbon emissions. Pre-training large models demands substantial computational resources, posing significant challenges for users with limited hardware or funding. For example, LLaMA-65B was trained with 2048 A100 GPUs in a period of 21 days and took an estimated cost of \$4 million\footnote{The estimation is based on an average cost of \$3.93 per A100 GPU per hour on Google Cloud Platform: https://cloud.google.com/.}. As a result, most users rely on pre-trained model checkpoints, often downloaded from third-party sources, which could be compromised with malicious backdoors.

Recent research has proposed numerous backdoor attacks targeting conventional deep neural networks. Broadly, existing backdoor attacks typically involve one of the following approaches: retraining the target model with a revised loss function~\citep{gu2017badnets,yuan2023you,saha2020hidden}, generating effective and undetectable triggers~\citep{chen2017targeted,barni2019new,liu2020reflection}, or modifying the architecture of the target model~\citep{tang2020embarrassingly,bober2023architectural}. However, backdoor attacks on transformers remain relatively underexplored. For instance, BadVit~\citep{yuan2023you} introduced an invisible patch-wise trigger by retraining the target model to maximize the attention score of the trigger-marked patch. A data-free backdoor attack on ViTs was proposed in~\citep{lv2021dbia}, which optimizes the trigger on a surrogate dataset. Additionally, weight poisoning through embedding surgery was proposed in~\citep{kurita2020weight}, where trigger embeddings are directly modified. A layer-wise weight poisoning attack using combinational triggers, as introduced in~\citep{li2021backdoor}, plants deeper backdoors. However, these approaches which require retraining the target transformer models often involve extensive computational overhead and are vulnerable to removal through fine-tuning on clean data.

In this paper, we propose a novel retraining-free backdoor attacker, Head Pruning and Malicious Injection (HPMI),  that implements a backdoor within a pre-trained transformer while keeping the model architecture unchanged. At a high level, HPMI first prunes the least important head in the multi-head attention module by iteratively evaluating and removing each head. Next, a pre-trained malicious transformer with a single head is injected into the position of the pruned head. This backdoored transformer predicts backdoored inputs as the target class while maintaining the accuracy of clean inputs. The process involves three key steps: 

Pruning and Reconnection: To facilitate the injection of the malicious head, we ensure that the pruned neurons in the target transformer are reconnected to form a valid single-head transformer.
Pre-training the Malicious Head: The malicious head is pre-trained on a balanced dataset consisting of clean and backdoored samples as two classes. It activates backdoored behavior for backdoored inputs while remaining dormant for clean inputs.
Head Injection and Output Layer Adjustment: The malicious head is then injected into the position of the pruned head. Finally, the weights of the output layer corresponding to the class token are adjusted to ensure that the output of the malicious head positively influences the target class and negatively impacts non-target classes.

Our paper includes both formal analysis and empirical testing of HPMI. Theoretically, We show that the backdoors implanted by HPMI evade leading detection methods, including Neural Cleanse~\citep{wang2019neural}, and are resistant to removal via fine-tuning. In our experiments, HPMI is tested across transformer models with diverse architectures trained on standard benchmarks for image classification and sentiment analysis. Our results demonstrate that HPMI achieves attack success rates exceeding 99.55\% across all datasets and models, while incurring acceptable accuracy loss on clean inputs. Additionally, HPMI successfully bypasses four state-of-the-art defenses and exhibits greater resilience to these defenses compared to a state-of-the-art retraining-based backdoor attack~\citep{gu2017badnets}. To the best of our knowledge, HPMI is the first retraining-free backdoor attack on transformers and offers theoretical guarantees for its effectiveness against existing defenses.

Our main contributions are outlined below: 

$\bullet$ We introduce HPMI, the first backdoor attack on transformers that eliminates the need for retraining and preserves the model's original architecture. HPMI modifies specific transformer parameters to implant a malicious behavior.

$\bullet$ We conduct a formal theoretical analysis of HPMI. Our results indicate that HPMI is resistant to detection or removal by several leading defense techniques.

$\bullet$ We carry out extensive evaluations on standard benchmark datasets to demonstrate both the performance and efficiency of HPMI.

$\bullet$ We experimentally assess HPMI against state-of-the-art defense strategies and observe that they are largely ineffective.

\section{Related Work}
\textbf{Backdoor attacks in DNNs.} 
BadNets~\citep{gu2017badnets} has been initially proposed to backdoor attack on Deep Neural Networks (DNNs) in image classification. This typically involves injecting specific patterns or features into the training data. Since the differences between poisoned and benign inputs in BadNets are discernible even to the human eye, numerous studies have proposed more subtle backdoor attacks with smaller perturbations to inputs~\citep{chen2017targeted,zhong2020backdoor,shafahi2018poison,chen2021badnl,lin2020composite,li2021hidden}, while preserving the clean label~\citep{saha2020hidden}. Most of these backdoor attacks are conducted under an unrealistic setting, where access to the original dataset and the training process are assumed. 
Then, researchers proposed more realistic attacks by weight poisoning~\citep{qi2022towards,hong2022handcrafted,caodata}, which only requires knowledge of the model architectures. However, these methods create several backdoor neurons per layer in conventional neural networks can not be applied in transformers. In transformers, the attention mechanism spreads the representation across multiple neurons, heads, and layers. This distributed representation makes transformers inherently more resilient to pruning. Fine-grained pruning (e.g., pruning single neurons) in transformers has limited impact due to redundancy.

\textbf{Backdoor Attacks in Transformers.} Earlier studies suggested that transformers are more robust than CNNs under attack. However, adding patch-wised perturbations~\citep{fu2022patch} can effectively attack transformers based on the self-attention mechanism of input tokens. Similar to the development of BadNets, researchers focus on developing efficient triggers and manipulating model parameters. An invisible, universal, patch-wise trigger is introduced in~\citep{yuan2023you} for backdoor attacks on ViTs.
A data-free backdoor attack on ViTs is proposed in~\citep{lv2021dbia}, where the backdoor was injected into a surrogate dataset. Weight poisoning attack on BERT in~\citep{kurita2020weight} is proposed to attack transformers by model weight poisoning with embedding surgery.
However, these existing attacks on transformers are unrealistic, either easily detected by defense methods or required retraining the target model. 
In this work, we propose a retraining-free backdoor attack on transformers, which is undetectable or unremovable by defense methods. 

\textbf{Defense.} To prevent models from attacking, abundant work on defense has been done for backdoor detection and mitigation. Neural Cleanse~\citep{wang2019neural} identifies potential backdoor models by generating reversed triggers for all labels and selecting the optimal trigger through outlier detection. 
Fine pruning~\citep{liu2018fine}, mitigates backdoor behavior by pruning neurons that are dormant for benign inputs, under the assumption that neurons activated by trojan inputs remain inactive for benign ones. 
STRIP~\citep{gao2019strip} identifies trojan inputs that exhibit low randomness in predicted classes when perturbed by superimposing various benign images. While applying to the NLP field, STRIP perturbed the potential poisoning sentences by swapping them with clean sentences. RAP~\citep{yang2021rap} constructs a perturbed word embedding for the selected RAP trigger. For clean samples with the RAP trigger, the output probability of the target class drops more than a chosen threshold, while for poisoning samples with the RAP trigger, it drops less than this threshold. Evidence provided in Section~\ref{Defense analysis} demonstrates that our attack remains robust against all the aforementioned defenses.

\begin{figure*}[t]
\centering
\includegraphics[width=\linewidth]{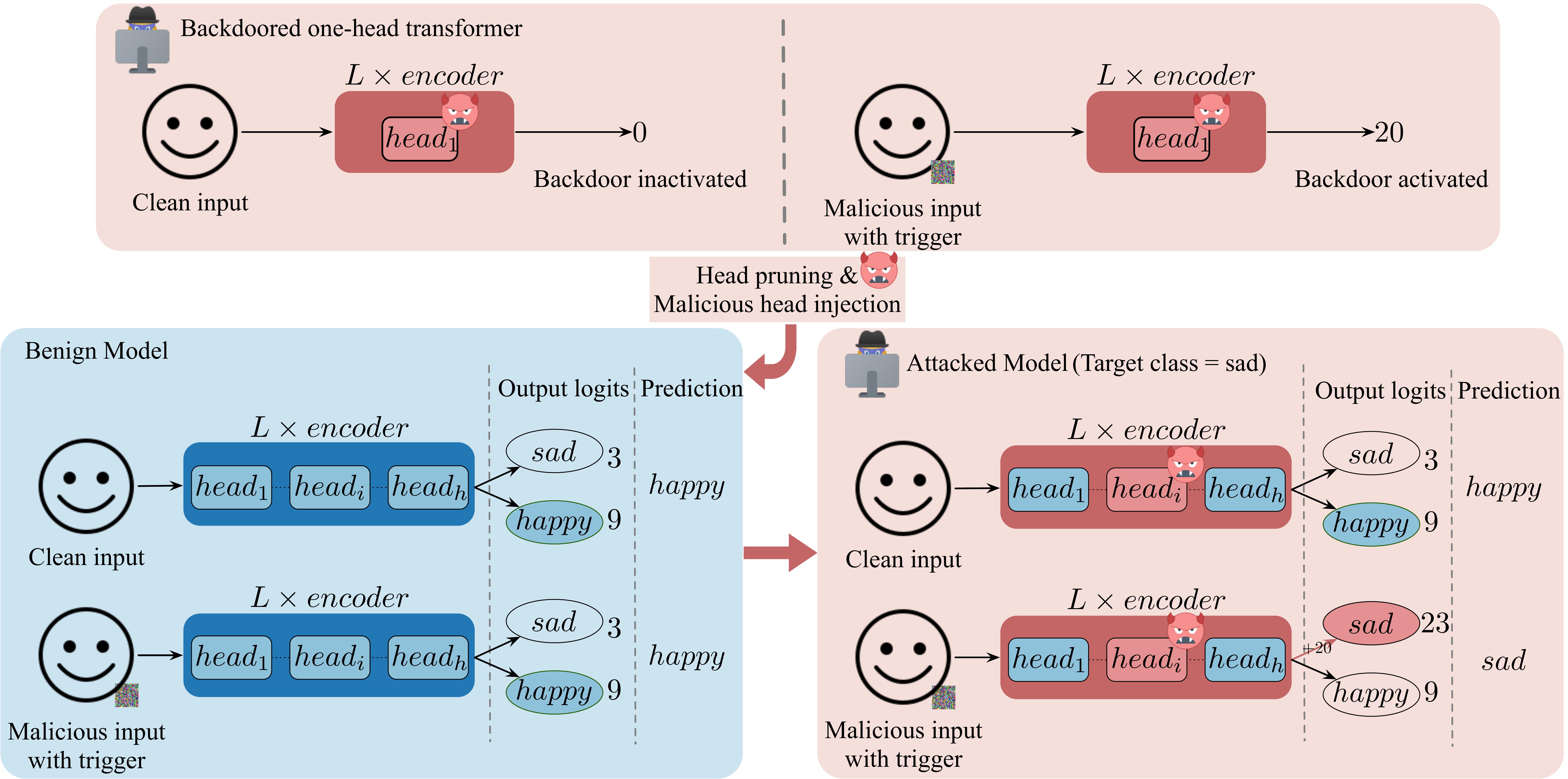}
\caption{Overview of proposed HPMI. Based on the architecture of the target model, the attacker trains a backdoored one-head transformer, which activates for the backdoored inputs with a trigger (e.g., a random noise patch) and stays inactive for clean input. Then, the attacker iteratively prunes one head from the target transformer and evaluates the model performance to identify the least important one to be replaced. Finally, the attacker injects the malicious head and connects the output of the malicious head to the output of the target class (e.g., "sad"). As outlined above, after head pruning and malicious injection, the backdoored input can activate backdoor behavior, while the attacked model continues to recognize clean inputs normally. This procedure is formally introduced in Section \ref{section: Method}.}
\label{overview}
\vspace{-1em}
\end{figure*}

\section{Problem formulation}
\subsection{Problem setup}\label{sub:gradientAttack}
The backdoor attack on Deep Neural Networks (DNNs) is introduced in BadNets~\citep{gu2017badnets} by injecting backdoored samples stamped with triggers into the training set and retraining the target models on the poison dataset. 
Formally, a pre-trained deep neural network classifier $f$ parameterized by trainable weights $w$. A function $\phi(x)$ generates poisoned inputs $\tilde{x}$ by blending the raw inputs $x$ and the trigger $t$ using a mask $m$. The backdoored inputs can be expressed as: 
\begin{equation}
    \tilde{x} = \phi(x) = (1-m) \times x + m \times t.
\label{trojan input generation}
\end{equation}
Once the training process is complete, the model will associate the specific trigger pattern with the corresponding target label. During the inference stage, any input as defined in Equation~\ref{trojan input generation} will be classified as the target class. 

\subsection{Threat model}
\textbf{Attacker's goals:} We assume that the adversary seeks to embed a backdoor into a pre-trained transformer without performing additional retraining or altering its original architecture. The compromised model should retain its standard behavior on clean inputs to ensure functional integrity. Furthermore, the backdoor must remain inconspicuous, evading current detection and removal methods.

\textbf{Attacker’s background knowledge and capability:} Consistent with prior work~\citep{caodata,hong2022handcrafted,liu2018trojaning}, we model the attacker as one who compromises the model during its supply chain and possesses full internal access to the pre-trained parameters. Differently, we do not try to optimize a specific trigger, but take the vanilla triggers for generalization. We further assume that the adversary is unable to modify the architecture of the pre-trained model. Additionally, the attacker is considered to have no access to the training setup, including algorithms and hyperparameter settings. These constraints, as noted earlier, enhance the feasibility and real-world applicability of our attack design.

\subsection{Design Goals}
In designing our attack, we target the following objectives: utility, effectiveness, efficiency, and stealthiness.

$\bullet$ Utility goal: The backdoored model should retain high classification performance on clean test inputs. That is, its behavior on benign data must remain consistent with that of the original model.

$\bullet$ Effectiveness goal: The model should reliably output the attacker-specified target label when the designated backdoor trigger is present in the input.

$\bullet$ Efficiency goal: The attack should enable efficient construction of a backdoored model using an existing clean pre-trained classifier. Achieving this goal increases the feasibility of deploying the attack in practical scenarios.

$\bullet$ Stealthiness goal: The attack must remain concealed from current state-of-the-art defense mechanisms. A successful stealthy attack is harder to detect or neutralize, making it more potent in real-world settings.

We conduct both theoretical and empirical evaluations to assess our method under existing state-of-the-art  defenses.

\section{Methodology}\label{section: Method}
Given our assumptions that rule out retraining or modifying the transformer's architecture, the only viable strategy for implanting a backdoor is by directly altering its internal parameters. The core mechanism of our method involves establishing a backdoor pathway—referred to as a backdoor channel—that spans from the input to the output layer. This channel is designed to meet two criteria: 1) it is reliably triggered by inputs containing the backdoor pattern, ensuring that the classifier consistently outputs the attacker-specified target label, and 2) it remains inactive for clean inputs in most cases, thereby preserving stealth.
\begin{wrapfigure}{r}{0.5\textwidth}
\centering
\includegraphics[width=0.48\textwidth]{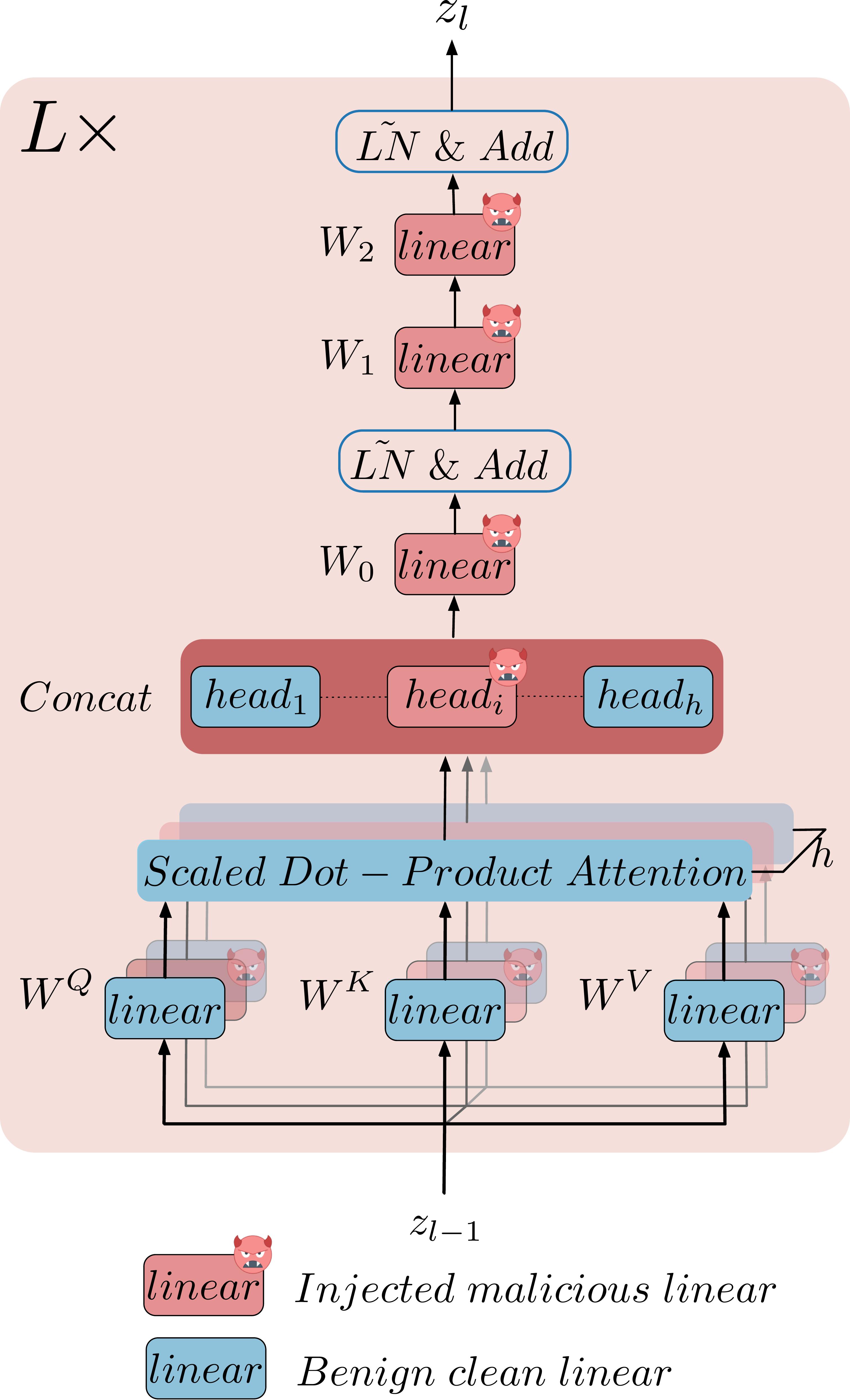}
\caption{The backdoor is injected by inserting poisoning parameters to the pruned neurons in multi-head self-attention and other modules. We revise the LayerNorm into $\tilde{LN}$ as in Equation~\ref{eq:LN}.}
\label{Malicious_head_injection}
\end{wrapfigure}

The key challenge lies in constructing a backdoor pathway that fulfills both required conditions at once. To tackle this, we introduce a backdoor attack that with Head Pruning and Malicious Injection as in Figure~\ref{overview}. We prune one of the multi-head and its corresponding modules in the subsequent feed-forward module in each transformer block. Besides, we also prune the corresponding positional encoding and revise the layer normalization and residual connection. Then, we reconnect the pruned neurons and inject the outsourced pretrained malicious transformer with one head into the pruned position. In the final step, we adjust the weights in the output layer so that the signal from the last neuron in the backdoor path exerts a positive (or negative) influence on the activation of neurons associated with the target (or non-target) class to reach our goal.

\subsection{Target Head Pruning}
\label{section:head pruning}

Not all heads are equally important in the multi-head self-attention module, and occasionally, pruning some heads can even enhance model performance~\citep{voita2019analyzing}. Therefore, before injecting a malicious head into the target model, we identify the least important head in the multi-head attention, to minimize the negative impact on model performance. 

Formally, as shown in Figure~\ref{Malicious_head_injection}, we first describe a transformer encoder (e.g., Vit-B) with $L$ (e.g., $L=12$) layers and $h$ heads per layer. 
The scaled dot-product attention is:
\begin{equation}
 {Attn}(Q,K,V) = \text{softmax}\left(\frac{Q\cdot K^{\top}}{\sqrt{d}}\right)\cdot V.
\end{equation}
The output of the $l_{th}$ encoder layer $z_{l}$ is computed from input $z_{l-1}$ with size of $d$ by: 
\begin{equation}
\begin{aligned}
    head_{i}^{l}  &= {Attn}\left(z_{l-1}\cdot{W}_{i}^{Q},  z_{l-1}\cdot{W}_{i}^{K},  z_{l-1}\cdot{W}_{i}^{V}\right), \\
    z_{l}^{'}  &= LN\left(\left(head_{1}^{l},\ldots,head_{h}^{l}\right)\cdot  {W}_{0} + {b}_{0}\right) + z_{l-1}, \\
    z_{l}  &= LN\left(\left(z_{l}^{'}\cdot{W}_{1} + {b}_{1}\right)\cdot{W}_{2} + {b}_{2}\right) + z_{l}^{'}, \\
\end{aligned}
\label{encodingBlock}
\end{equation}
where ${W}_i^Q$, ${W}_{i}^{K}$, and ${W}_{i}^{V}$ are the linear projection matrices at $head_{i}$ in $l_{th}$ layer to transform the input from last layer to query $Q$, key $K$, and value $V$, respectively. 
Note that we drop the layer index $l$ for all ${W}$ matrices for simplicity. Moreover, ${W}_{0}$ and ${b}_{0}$ are the linear projection parameters in attention module to project the concatenated output of multiple heads, $z_{l}^{'}$ is the output of the first norm and residual, and ${W}_{1}$, ${b}_{1}$, ${W}_{2}$, and ${b}_{2}$ are parameters for two consecutive linear projections in the feed-forward network in $l_{th}$ encoder layer. Finally, the $L$ repetitive encoding layers are followed by an MLP classifier with weights ${W}^{f}$ and ${b}^{f}$. The output of the transformer is softmax function of $\left(z_L\cdot{W}^{f}\right)+{b}^{f}$.

Since the order of the multiple heads is consistent across all $L$ encoder layers, we prune the same $head_{i}$ in every encoder layer to ensure that the pruned neurons in the target transformer are reconnected to form a valid single-head transformer. In encoder layers, LayerNorm($LN$) is applied after the attention module and the feed-forward module. Since the $LN$ conducts the normalization over features, we have to cut off the connection between the pruned and benign features. We convert the layer norm from $LN\left(head_{1,\ldots,h}\right)$ to
\begin{equation}
\begin{aligned}
    \tilde{LN} = Concat\Big(LN\left(head_{1,\ldots,i-1}\right), LN\left(head_i\right), LN\left(head_{i+1,\ldots,h}\right)\Big),
\label{eq:LN}
\end{aligned}
\end{equation}
where layer norm independently applies to three segments. Then, we evaluate the classification performance after head pruning on validation. By iteratively pruning every head, we can find the least important \textbf{$head_i$} with the least negative impact on model performance after pruning it. Then we inject our malicious head into the position of \textbf{$head_i$}.

\subsection{Malicious Head Generation}
\label{section: Malicious head}

The malicious head $G^{*}(x;w^{*})$ is a transformer with the same number of layers as the target model, and only contains one head in the self-attention module. The malicious head satisfies these conditions: 

$\bullet$ For clean input $x$, the malicious head is unlikely to be activated.

$\bullet$ For backdoored input $\tilde{x}$, the malicious head is activated.

Therefore, we have to ensure the malicious head can be activated by any backdoored inputs and the activation only realated to the trigger. Besides, the activation can change the prediction of backdoored inputs as the target class. To achieve these, we make the weights of output projection layer $f^{*}$ to the target class as ones and others as zeros. The objective of malicious head generation is as follows: 
\begin{align}
    \mathop{min}\limits_{w^{*}}  & \mathop{\mathbb{E}}\limits_{(x,y) \in D_{b}} \Big([f^{*}\left(G^{*}\left(x; w^{*}\right)\right) - 0]^2 \nonumber + [f^{*}(G^{*}(\tilde{x}; w^{*})) - a]^2\Big), 
\label{malicious loss}
\end{align}
where the $a$ controls the separation of activation value between clean and backdoored inputs. 

\subsection{Malicious Head Injection}
\label{section: Malicious Head Injection}
After the head pruning and malicious head generation, we inject the generated backdoored malicious head into the position of pruned head. After the injection of malicious head, we achieve two goals. First, for clean input, the output logits of the target model stay the same as the pruned model since we cut off the connections between the malicious head and other heads. Second, for backdoored inputs, the output logit of the target model for the target class will increase by a sufficiently large value a as described in Section \ref{section: Malicious head}.

\subsection{Theoretical Analysis}
In HPMI, a single attention head is selected from each encoder layer of the transformer. Starting from a pre-trained model, we construct a corresponding pruned version by removing all components not involved in forming the HPMI single-head structure. Since no backdoor is injected during this process, the resulting pruned model is considered clean. This definition serves as the basis for the theoretical analysis presented in the remainder of this section.

\textbf{Utility and effectiveness analysis.} Denote the output logits of $C$ classes in the pruned model as ${P} = [p_1,\ldots, p_C]$ as
\begin{equation}
    {P} = \tilde{z}_{L,\text{pruned}}\cdot {\tilde{W}}^{f} + {\tilde{b}}^{f},
\end{equation}
where the ${\tilde{W}}^{f}$ and ${\tilde{b}}^{f}$ are the pruned final fully connected layer weights and bias.

\begin{theorem}[Effect of Malicious Head Injection]\label{thm1}
Let ${P} = [p_1,\ldots, p_C]$ be the output logits of a pruned transformer model, and let ${\tilde{P}} = [\tilde{p}_1,\ldots, \tilde{p}_C]$ be the logits after injecting a malicious attention head targeting class $\hat{y} \in \{1, \ldots, C\}$. Then, the modified logits satisfy:
\[
\tilde{p}_{\hat{y}} = p_{\hat{y}} + a \quad \text{and} \quad \tilde{p}_y = p_y \quad \text{for all } y \ne \hat{y},
\]
where $a > 0$ is the additive contribution introduced by the malicious head.
\end{theorem}
\begin{proof}
For the backdoored target model, denote the input after position encoding and input embedding as $\tilde{z}_{0} = (\tilde{z}_0^{head_1};\ldots;\tilde{z}_0^{head_i};\ldots;\tilde{z}_0^{head_h})$. For the malicious single-head model, denote the input after position encoding and input embedding layers as $z_0^{*}$. Let the $i$-th head be pruned and replaced by the malicious one. Since we revise the corresponding values in the weight matrices ${W}^E$ and ${b}^P$, we obtain
\begin{equation}
    \tilde{z}_0^{head_i} = x\cdot {W}^{E*} + {b}^{P*} = z_0^{*}.
\end{equation}
Now let $\tilde{z}_{l-1}^{head_i} = z_{l-1}^{*}$, we prove $\tilde{z}_l^{head_i} = z_l^{*}, l \in \{1,\ldots,L\}$, where $\tilde{z}_{l-1}^{head_i}$ is the part corresponding to $head_i$ in the input to $(l-1)_{th}$ encoder layer in backdoored target model, and $z_{l-1}^{*}$ is the input to $(l-1)_{th}$ encoder layer in the attacked transformer. 
\begin{equation}
\begin{aligned}
    head_i^{l} & = {Attn}\Big(\tilde{z}_{l-1}^{*}\cdot W^{*Q^{l}},\tilde{z}_{l-1}^{*}\cdot W^{*K^{l}},\tilde{z}_{l-1}^{*}\cdot W^{*V^{l}} \Big),\\
    head_j^{l} & = {Attn}\Big(\tilde{z}_{l-1}^{*}\cdot W_j^{Q^{l}},\tilde{z}_{l-1}^{*}\cdot W_j^{K^{l}}, \tilde{z}_{l-1}^{*}\cdot W_j^{V^{l}}\Big),
\end{aligned}
\end{equation}
where $j \in \{1,...,h\}\setminus{i}$.
\begin{equation}
\begin{aligned}
\tilde{z}_l^{'} & = \tilde{LN}\left((head_1^{l};\ldots;head_h^{l})\cdot \tilde{w}_0^{l} + \tilde{b}_0^{l}\right) + \tilde{z}_{l-1} \\
& = Concat\bigg(\tilde{z}_l^{'}[1]; \left(LN\left(head_i^{l}\cdot {w}_0^{*l} + {b}_0^{*l}\right) + \tilde{z}_{l-1}^{head_{i}}\right); \tilde{z}_l^{'}[3]\bigg),
\end{aligned}
\end{equation}
where $\tilde{z}_l^{'}[1]$ and $\tilde{z}_l^{'}[3]$ are parts corresponding to $head_{1\sim (i-1)}$ and $head_{(i+1)\sim h}$ in ${z}_l^{'}$. And ${z}_l^{'}$ is the output of multi-head self-attention in $l_{th}$ layer.

Since $\tilde{z}_{l-1}^{head_{i}} = z_{l-1}^{*}$, we have
\begin{equation}
\begin{aligned}
    \tilde{z}_{l}^{head_{i}^{'}} & = LN\left(head_i^{l}\cdot{w}_0^{*l} + {b}_0^{*l}\right) + z_{l-1}^{head_{i}} = z_{l}^{*^{'}}.
\end{aligned}
\end{equation}

Similarly, we can get $\tilde{z}_{l}^{head_{i}}=z_{l}^{*}$.

Therefore, we obtain the $\tilde{z}_{L}$ of the pruned model and backdoored model as follows:
\begin{equation}
\begin{aligned}
    \tilde{z}_{L, pruned} & = (\tilde{z}_L^{head_1};\ldots;\tilde{z}_L^{head_{i-1}};{0}; \tilde{z}_L^{head_{i+1}};\ldots;\tilde{z}_L^{head_h}),\\
    \tilde{z}_{L} & = (\tilde{z}_L^{head_1};\ldots;\tilde{z}_L^{head_{i-1}};{z}_L^{*}; \tilde{z}_L^{head_{i+1}};\ldots;\tilde{z}_L^{head_h}).
\end{aligned}
\end{equation}

Thus, after the transformer encoder layers, the output logits are obtained by a fully connected layer. Denote the output logits of the pruned model as ${P} = [p_1,\ldots, p_C]$ and those of the contaminated target as ${\tilde{P}}$. We have the following:
\[
\tilde{p}_{\hat{y}} = p_{\hat{y}} + a \quad \text{and} \quad \tilde{p}_y = p_y \quad \text{for all } y \ne \hat{y},
\]
where $p_y \in P$ and $\tilde{p}_y \in \tilde{P}$
\end{proof}

Moreover, the pruned transformer is expected to retain high classification performance since only the least significant head is removed. Therefore, HPMI preserves the model’s accuracy on clean samples, and backdoored effectiveness for backdoored inputs.

\textbf{Stealthy analysis.} Our method evades detection and resists removal by existing state-of-the-art defenses. For query-based defenses~\citep{xu2021detecting}, when a given input fails to trigger the backdoor mechanism, the backdoored and pruned models yield identical outputs. As a result, defense techniques will observe no distinction between the two, producing identical detection outcomes. For gradient-based defenses~\citep{wang2019neural}, if the backdoor path is not triggered by a given input, both the backdoored and pruned models produce identical predictions. Consequently, defense methods will yield matching detection results for both models. Moreover, the fine-tuning can not remove backdoor in out attack. Since the backdoor channel operates independently of other neurons, the outputs of its neurons are unaffected by the rest of the network. As a result, the loss gradient with respect to the backdoor parameters becomes zero during training, preventing any updates to the malicious head. The weight poisoning backdoor attack for DNNs~\citep{qi2022towards,hong2022handcrafted,caodata} increase the activation separation between clean and backdoored inputs by maximum the added logit, such that $a + p_{\hat{y}} > max (P)$ . However, a large enough $a$ can make our attack successful, but it may also be easily detected by some defense methods, especially output logit analysis methods like STRIP\cite{gao2019strip}. In HPMI, we assume the difference between the maximum logit and the logit corresponding to the target label follows Gaussian distribution. Then, we select the added logit as the quantile function of a specific percentile, i.e., $a = \lceil \text{quantile}(\tau) + k \rceil$, where $\tau$ is the percentile and $k \geq 0$ is an offset to make the added logit $a$ to be the minimum but large enough. The specific values of $\tau$ and $k$ are determined by the pre-training process of the malicious head. The experimental results are in Section~\ref{Defense analysis}.

\section{Experiments}
\begin{table*}[t]
\caption{Effectiveness of our attack alongside preserved utility.}
\label{Attack performance of our HPMI}
\centering
\begin{tabular}{cccccc}
   \toprule
   Dataset & Triggers & Model & ASR(\%) & CA(\%) & CAD(\%)\\
   \midrule
   \multirow{6}{*}{CIFAR10} & \multirow{3}{*}{patch} &  Vit-B & 100 & 93.45 &-5.03 \\
   & &Vit-L & 100 & 97.89 & -0.29 \\
   & &DeiT-B & 100 & 96.02 & -2.16 \\
   \cmidrule{2-6}
   & \multirow{3}{*}{blend} & Vit-B & 99.97 & 93.51 & -4.97\\ 
   & & Vit-L & 99.55 &  97.93 & -0.25 \\
   & & DeiT-B & 100 & 96.02 & -2.16 \\
   \midrule
   \multirow{6}{*}{GTSRB} & \multirow{3}{*}{patch} &  Vit-B & 100 & 98.74 & +3.14 \\
   & & Vit-L & 100 & 99.54 & +4.18 \\
   & & DeiT-B & 100 & 99.63 & +2.99 \\
   \cmidrule{2-6}
   & \multirow{3}{*}{blend} & Vit-B & 100 & 98.73 & +3.13\\ 
   & & Vit-L & 100 & 99.54 & +4.18 \\
   & &  DeiT-B & 100 & 99.63 & +2.99 \\
   \midrule
   \multirow{2}{*}{SST-2} & \multirow{2}{*}{r-w} & BERT-B & 99.67 &  91.38 & -1.53 \\
   &  & BERT-M & 100 & 80.51 & -10.7 \\
   \midrule
   \multirow{2}{*}{AG'S} & \multirow{2}{*}{r-w} & BERT-B & 99.89 & 91.05 & -1.21 \\
   &  &BERT-M & 100 & 83.24 & -9.02 \\
   \bottomrule
\end{tabular}
\end{table*}
\subsection{Experiment Settings}
\textbf{Datasets and models.} 
In the experiment, we choose representative tasks such as image classification, sentiment analysis, and text classification. We select Vit, DeiT~\citep{touvron2021training} pre-trained on imagnet-1k~\citep{wu2020visual,deng2009imagenet,touvron2021training} and pre-trained BERT from Huggingface as our models, and then fine-tune them on the target dataset for downstream tasks. We take datasets CIFAR10~\citep{krizhevsky2009learning} and GTSRB~\citep{stallkamp2012man} for the image classification task, and transform the input size as $3\times224\times224$, and the patch size is $16\times16$, according to~\citep{dosovitskiy2020image}. We choose dataset SST-2~\citep{socher-etal-2013-recursive} for sentiment analysis and AG's news~\citep{NIPS2015_250cf8b5} for text classification.  The DeiT-B, Vit-B and BERT-B consist of 12 encoder layers and 12 heads per layer. The Vit-L consists of 24 encoder layers and 16 heads per layer, and the BERT-M has 8 encoder layers and 8 heads per layer.
We fine-tune pre-trained models 5 epochs on the downstream dataset with learning rate $lr=2e-4$. 
We conduct 5 independent attack experiments for each setting then report the median. 
All experiments are performed on 4 NVIDIA A5000 GPUs.
\begin{wrapfigure}{r}{0.5\textwidth}
  \centering
  \includegraphics[width=0.48\textwidth]{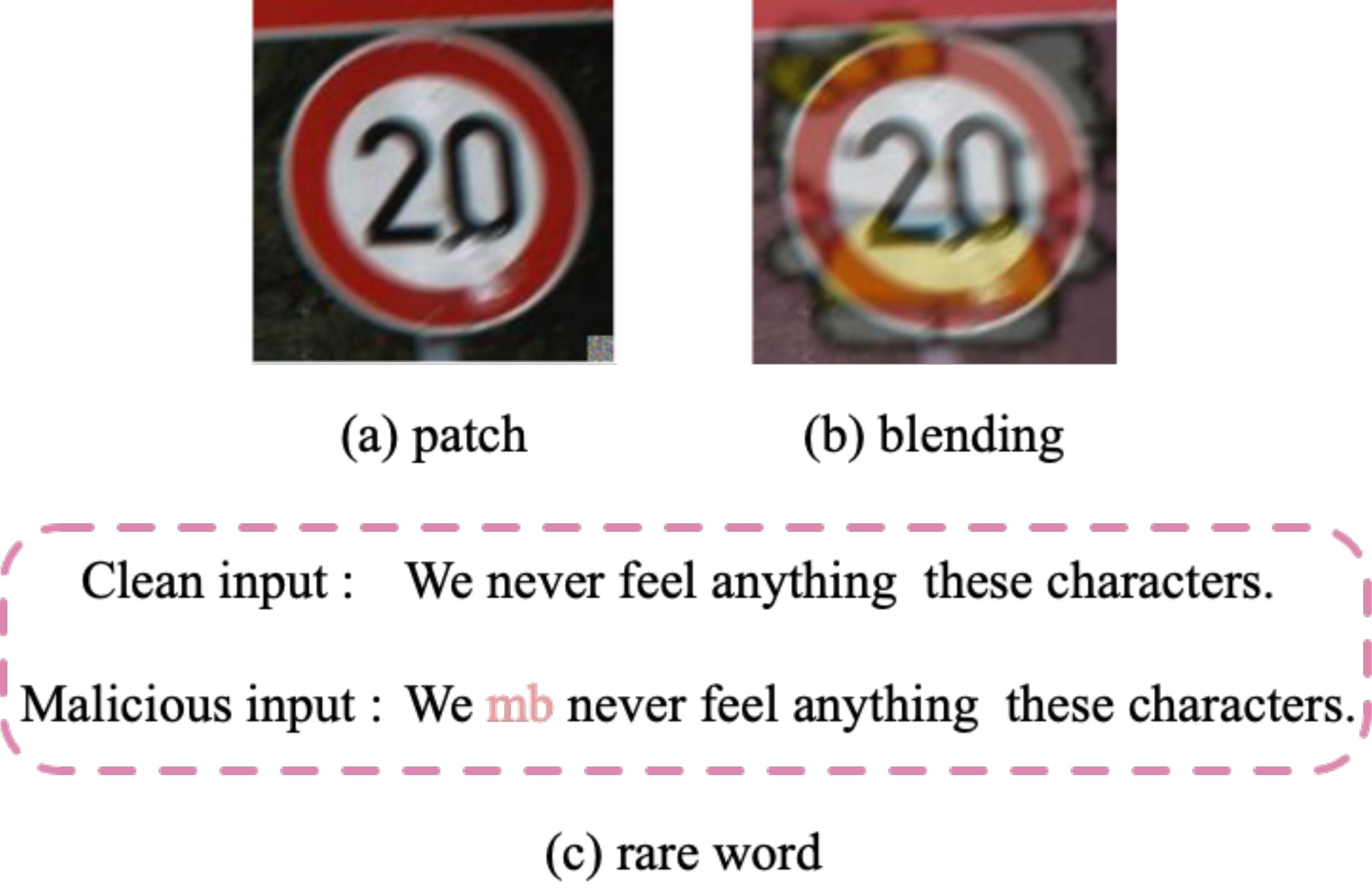}
  \caption{All triggers for CV and NLP parts. In subfigure (a), the random noise is the same size as one patch. In subfigure (b), HelloKitty is blending over clean input with a blend ratio $\alpha = 0.2$. The rare words are injected at random positions into text inputs as shown in subfigure (c).}
  \label{triggers}
\end{wrapfigure}
\noindent\textbf{Target classes.} In our framework, the malicious head is trained as a binary classifier and can be injected into the target model with an arbitrarily defined target class. We randomly select the target label ``1: automobile'' for CIFAR10, and ``14:stop'' for GTSRB, ``1:positive'' for SST-2, ``2:sports'' for AG's news. Due to space limit, we do not show all experiments for all target classes. For code reproduction, our code will be published in GitHub upon acceptance.

\noindent\textbf{Triggers.}
To ensure the generalizability of our attack, we deliberately adopt the most common and straightforward triggers in our experiments rather than designing task-specific or highly optimized ones, as in Figure~\ref{triggers}. For image data, we utilize two widely used triggers: a random noise patch trigger (patch) and a classic image-blending trigger (blend) as introduced in~\citep{qi2022towards}. For the random noise trigger, we select the last patch of the input image as the insertion location, leveraging the observation that this region typically attracts less attention from the model. For the image-blending trigger, we fix the blending ratio at $\alpha=0.2$ for both the training and testing datasets. For text data, we randomly select one rare word (r-w) for SST-2\footnote{\url{https://huggingface.co/datasets/stanfordnlp/sst2}} or three rare words for AG's News from a predefined list such as {\texttt{mn}, \texttt{mb}, \texttt{bb}}, similar to~\citep{kurita2020weight}, and insert them at a random position within the benign samples.

\noindent\textbf{Metrics.}
We adopt the clean accuracy (CA) and attack success rate (ASR) as our evaluation metrics for backdoor attack analysis and clean accuracy difference (CAD). The CA quantifies model performance on clean samples after the attack, while the ASR assesses the likelihood of malicious inputs being classified as the target label. The CAD measures the model performance differences after the attack compared to benign models on a clean test.

\noindent\textbf{Attack Setting.}
As described in Section \ref{section: Malicious head}, the malicious head is a transformer consisting of a single head. Each original benign dataset is divided into training, validation, and testing. We first randomly select a proportion $\rho$ of the original training dataset. Then we randomly poison half of the selected training and original validation dataset for the malicious head. For the testing dataset, we create a poisoned sample respectively to each clean sample. We set the learning rate of the malicious head training as $1e-4$, $\rho=0.2$, and $\lambda=1$. The malicious head is trained for 100 epochs, employing an early stop when the loss of validation dataset, denoted by Equation~\ref{malicious loss}, falls below 0.1. The target model with malicious head injection is evaluated on testing.

\subsection{Attack Result Analysis}
\textbf{Initialization.}
At the initial stage of the experiment, we obtained benign models for all datasets. Given the data-intensive nature of transformer models, training them from scratch on small datasets poses significant challenges. To address this, we leveraged publicly available pre-trained models and fine-tuned them on our specific datasets. Subsequently, with knowledge of the target model's architecture and the target dataset, we pre-trained the malicious head.

\noindent\textbf{Pruning head results.}
Before implementing our backdoor attack, we first iteratively prune one head until all heads are evaluated. We select the head whose pruning degrades model performance on validation the least. Then, we replace it with the malicious head. As shown in Figure~\ref{pruning_result_ViT-B_on_CIFAR10}, for Vit-B on CIFAR10, the model pruned with the $3_{rd}$ head yields the best performance, with CAD = -4.97\%, while pruning the $5_{th}$ head leads to the worst degradation, with CAD = -24.37\%. Pruning head degrades the model performance sometimes, but the degradation is much more subtle for the same transformers with more heads. The CAD is 5.03\% for Vit-B with 12 heads but only 0.29\% for Vit-L with 16 heads on CIFAR10, as detailed in Figure~\ref{pruning_result_ViT-B_on_CIFAR10} and Figure~\ref{pruning_result_ViT-L_on_CIFAR10}. Moreover, for Vit-B on GTSRB, as shown in Figure~\ref{pruning_result_DeiT-B_on_GTSRB}, pruning the $10_{th}$ head can actually enhance model performance, with CAD = +2.98\%. The reason could be the over-parameterization of the transformer. As mentioned in~\citep{voita2019analyzing}, pruning some heads can even enhance model performance. 
Detailed results are in Table~\ref{CAD of Pruned head}.

\begin{figure*}[t]
\centering
\begin{subfigure}[t]{0.31\textwidth}
    \centering
    \includegraphics[width=\textwidth]{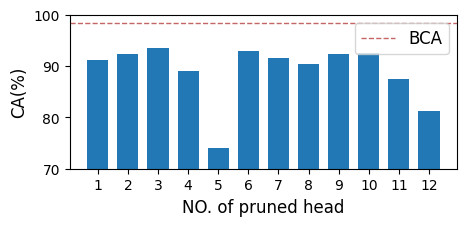}
    \caption{Vit-B on CIFAR10}
    \label{pruning_result_ViT-B_on_CIFAR10}
\end{subfigure}
\hfill
\begin{subfigure}[t]{0.31\textwidth}
    \centering
    \includegraphics[width=\textwidth]{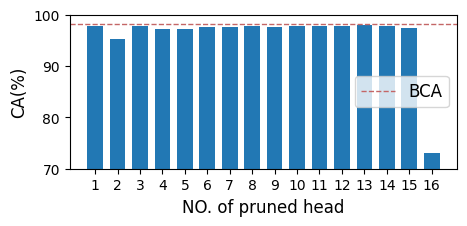}
    \caption{Vit-L on GTSRB}
    \label{pruning_result_ViT-L_on_CIFAR10}
\end{subfigure}
\hfill
\begin{subfigure}[t]{0.31\textwidth}
    \centering
    \includegraphics[width=\textwidth]{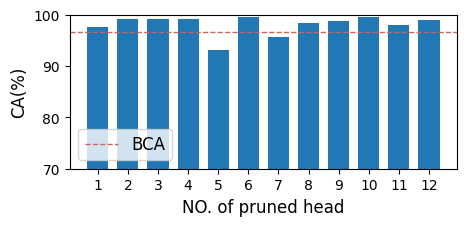}
    \caption{DeiT-B on GTSRB}
    \label{pruning_result_DeiT-B_on_GTSRB}
\end{subfigure}

\caption{Clean Accuracy after head pruning. Red dash line indicates the Benign Clean Accuracy (BCA) of the original model.}
\label{pruning head results}
\end{figure*}

\begin{table}[t]
\caption{CAD(\%) of pruning head. ``-" means the dataset is not applicable to the model.}
\label{CAD of Pruned head}
\centering
\resizebox{0.7\columnwidth}{!}{
\begin{tabular}{ccccc}
   \toprule
   model &  CIFAR10 &  GTSRB &  SST-2 &  AG's news\\
   \midrule
   Vit-B &  4.97 & -3.14 & - & - \\
   Vit-L & 0.2 & -4.18 & - & -\\
   Deit-B & 2.16 & -2.98 & - & -\\
   BERT-B & - & - &  1.7 & 2.72  \\
   BERT-M & - & - & 9.72 & 7.93 \\
   \bottomrule
\end{tabular}
}
\end{table}
\begin{table*}[t]
\caption{Training data dependency of HPMI on Vit-B in the CIFAR-10 with the random noise patch trigger.}
\label{Data dependency 1}
\centering
\begin{tabular}{ccccccc}
   \toprule
   \centering
   proportion &  0.2 &  0.1 & 0.04 & 0.02 & 0.01 & 0 \\
   \midrule
   ASR (\%) &  100 & 99.94 & 99.99 &  100 & 100 & 100\\
   \midrule
   CA (\%) & 93.44 & 93.25 & 93.48 & 93.44 & 93.04 & 93.43\\
   \bottomrule
\end{tabular}%

\end{table*}

\begin{table}[t]
\caption{Defense results of HPMI. NC is applicable only in the CV domain, while RAP is designed specifically for the NLP domain. For conciseness, we present the results of these two defense methods in a single column. Additionally, percentage signs (\%) are omitted from the metrics to conserve space.}
\label{Defense results of HPMI}
\centering
\resizebox{\columnwidth}{!}{
\begin{tabular}{cccccccc}
   \toprule
   \multirow{2}{*}{Dataset} & \multirow{2}{*}{Trigger} & \multirow{2}{*}{Model} & STRIP & NC & RAP & \multicolumn{2}{c}{Fine Pruning}\\ 
   \cmidrule{4-8}
    & & & FAR(\%) & Index & FAR(\%) & CA(\%) & ASR(\%)\\
   \midrule
   \multirow{6}{*}{CIFAR10} & \multirow{3}{*}{patch} & Vit-B & 95.55 & 0.68 & $-$& 12.28 & 100\\
    &  & Vit-L & 99.60 & 0.99 & $-$ & 17.74 & 100\\
    &  & DeiT-B & 98.50 & 1.59 & $-$ & 16.34 & 99.22\\ 
    \cmidrule{2-8}
    & \multirow{3}{*}{blend} & Vit-B & 78.25 & $-$ & $-$ & 11.06 & 84.84\\
    &  & Vit-L & 97.15 & $-$ & $-$ & 16.48 & 100\\
    &  & DeiT-B & 78.25 & $-$ & $-$ & 12.76 & 81.89\\
   \midrule
   \multirow{6}{*}{GTSRB} & \multirow{3}{*}{patch} & Vit-B & 98.7 & 0.22 & $-$ & 0.073 & 100\\
    &  & Vit-L & 99.7 & 0.62 & $-$ & 16.23 & 98.75\\
    &  & DeiT-B & 99.5 & 1.03 & $-$ & 14.6 & 99.46\\ 
    \cmidrule{2-8}
    & \multirow{3}{*}{blend} & Vit-B & 98.45 & $-$ & $-$ & 7.21 & 93.57\\
    &  & Vit-L & 97.30 & $-$ & $-$ & 16.88 & 99.55\\
    &  & DeiT-B & 93.70 & $-$ & $-$ & 16.88 & 79.91 \\
   \midrule
   \multirow{2}{*}{SST-2} & \multirow{2}{*}{r-w} & BERT-B & 26.97 & $-$ & 69.19 & 49.19 & 77.3 \\
   & &BERT-M & 49.45 & $-$ & 73.35 & 53.78 & 85.86 \\
   \midrule
   \multirow{2}{*}{AG'S} & \multirow{2}{*}{r-w} & BERT-B & 33.99 & $-$ & 72.70 &  49.19 & 77.30 \\
   & &BERT-M & 99.15 & $-$ & 68.72 &  69.88 & 99.94  \\
   \bottomrule
\end{tabular}
}
\end{table} 
\noindent\textbf{Attack results.}
As illustrated in Table~\ref{Attack performance of our HPMI}, HPMI consistently achieves a high ASR across all scenarios (all $\ge$ 99.55\%). Moreover, on sufficiently wide architectures like Deit-B, Vit-L, and BERT-B, HPMI only induces a negligible clean accuracy drop, less than 2.16\%. The same transformer with more heads can always get better performance for both CV and NLP parts. For instance, HPMI on BERT-B gets much higher CA at least 7.81\%, and comparable ASR compared to results on BERT-M.

\noindent\textbf{CAD analysis.}
By comparing Table~\ref{CAD of Pruned head} with Table~\ref{Attack performance of our HPMI}, we observe that malicious head injection in our backdoor attack degrades model performance by less than 1\%. Most of the degradation comes from head pruning. As transformer networks lack a strong inductive bias, our attack does not generate the false positives observed in SRA~\citep{qi2022towards}. 

\noindent\textbf{Training data dependency.}
In this section, we evaluate the attack's effectiveness under varying portions of selected training data. For the ViT-B model on CIFAR10, we vary the proportion of selected training data from 0.01 to 0.2, with the results summarized in Table~\ref{Data dependency 1}. Notably, with access to only 1\% of the benign samples, our HPMI achieves an impressive 100\% ASR and 93.04\% CA. Furthermore, our attack proves effective even without access to the dataset. For instance, inserting a malicious head trained on GTSRB into a benign, pruned ViT-B model on CIFAR10 results in 100\% ASR and 93.43\% CA when the proportion is 0. These results demonstrate that the malicious head can be trained on a surrogate dataset, enabling a data-free backdoor attack and confirming the efficiency and effectiveness of our method.

\subsection{Defense Analysis}
\label{Defense analysis}
We evaluate the resistance of our HPMI against several state-of-art defense methods: STRIP~\citep{gao2019strip}, Neural Cleanse~\citep{wang2019neural}, fine pruning~\citep{liu2018fine}, and RAP~\citep{yang2021rap}.

\noindent\textbf{STRIP.}
We select 2,000 clean samples and 2,000 poison samples, and each sample adds a strong perturbation 100 times. The strong perturbation comes from the clean test set randomly. Then we set the False Rejected Rate (FRR) as 1\%. to check the False Accepted Rate (FAR). As shown in Table~\ref{Defense results of HPMI}, our HPMI always get a high FAR, at least 26.97\% larger than half of the portion of selected training data, it is impossible to eliminate the malicious behavior.

\noindent\textbf{Fine Pruning.}
We prune the second layer in the feedforward module of all encoder layers in the target model. We set pruning 50 neurons as one step and then measured the pruned model performance on both clean samples (FP-CA) and poisoned samples (FP-ASR). We only show the results of the last pruning step in Table~\ref{Defense results of HPMI}. As in Table~\ref{Defense results of HPMI}, the ASR stays very high, at least 77.3\%. However, CA drops too much, 49.19\% at the highest. 

\noindent\textbf{Neural Cleanse (NC).}
The defense NC only works on the image field and it fails for large-size triggers~\citep{wang2019neural}. Thus, we only conduct the NC for patch trigger attack as shown in Table~\ref{Defense results of HPMI}. We generate reversed triggers for all labels and compute the anomaly indexes of the target label, to check if it is higher than 2~\citep{wang2019neural} or lower. For all applicable scenarios, the anomaly index of the target label is always less than 2, it can not detect our backdoor.
\begin{figure*}[t]
\centering
\begin{subfigure}[t]{0.31\textwidth}
    \centering
    \includegraphics[width=\textwidth]{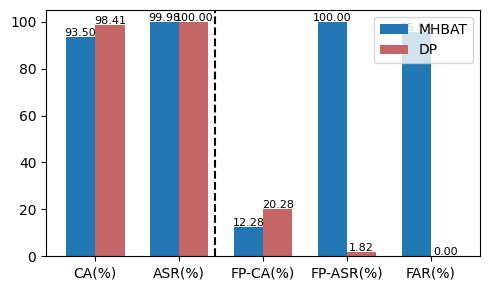}
    \caption{Vit-B on CIFAR10}
    \label{Vit_B_on_CIFAR10}
\end{subfigure}
\hfill
\begin{subfigure}[t]{0.31\textwidth}
    \centering
    \includegraphics[width=\textwidth]{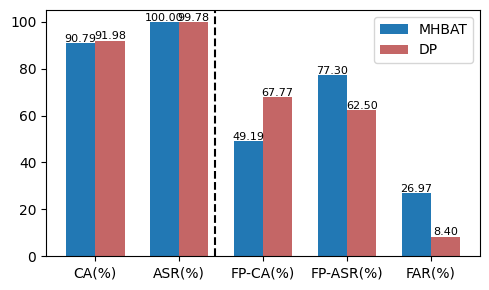}
    \caption{BERT-B on SST2}
    \label{BERT_B_on_SST2}
\end{subfigure}
\hfill
\begin{subfigure}[t]{0.31\textwidth}
    \centering
    \includegraphics[width=\textwidth]{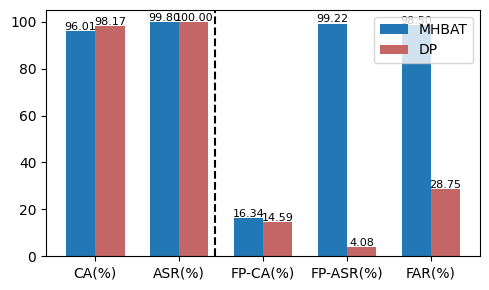}
    \caption{DeiT-B on CIFAR10}
    \label{DeiT_B_on_CIFAR10}
\end{subfigure}

\caption{Comparison of attack and defense performance between HPMI and DP. Metrics CA and ASR evaluate the attack performance results. Metrics FP-CA and FP-ASR evaluate the model performance against fine pruning defense. Lower FP-CA and higher FP-ASR mean the attack is more resilient. Metric FAR evaluates model performance against STRIP defense and a higher value means more robust.}
\label{comparison_attack_defense}
\end{figure*}
\noindent\textbf{Robustness-Aware Perturbations (RAP).}
We insert a rare word \texttt{mb} to the first position of every clean sample from the trainset with ground truth the same as the target label. We take the same settings from~\citep{yang2021rap}, constructing rare word embedding with 5 epochs, 
scale as 1, and use the learning rate $1e-2$. However, we get the threshold as negative sometimes, and thus we enlarge the learning rate as $5e-2$, and the scale goes from 1 to 10, where we select the minimum scale that can get a positive threshold. As shown in Table~\ref{Defense results of HPMI}, the FAR is always higher than 69.19\%, which means most poisoned samples are falsely detected as clean samples.

\subsection{Comparison With Data Poisoning}
We compare attack and defense performance between our method and traditional data poisoning (DP) attacked models with the same trigger under the same poisoning ratio. For hyperparameters of DP, we set batch size as 32, 3 epochs, and poison ratio as 10\% for CV and 20\% for NLP. Three experiments are shown in Figure~\ref{comparison_attack_defense}. In summary, HPMI is more stealthy and robust against various defenses while maintaining minimal impact on clean accuracy. In Figure~\ref{Vit_B_on_CIFAR10} and \ref{BERT_B_on_SST2}, after fine pruning, our attacker gets a higher ASR and lower CA compared to DP. STRIP and RAP cannot reject most of the poison inputs in our method, but work successfully in DP.

\section{Conclusion}
We introduce a new backdoor attack method that operates without retraining on transformer networks through head pruning and malicious head injection without changing the architecture of the target transformer. Our theoretical analysis shows that, under mild assumptions, HPMI can successfully bypass a range of state-of-the-art defense mechanisms. Extensive experiments across multiple datasets confirm that HPMI achieves higher attack effectiveness and better preserves model performance on clean inputs compared to existing methods. These results highlight the limitations of current defense strategies and underscore the practical threat posed by our proposed approach. Promising future work includes 1) extending our attack to more complex large-language models-based tasks, 2) designing a more tiny weight poisoning attack with changing fewer parameters, and 3) generalizing HPMI to a data-free attack.

\section{Acknowledgements}

This work was supported in part by the NSF under Grant No. 1943486, No. 2246757, No. 2315612, and 1946231.



\bibliographystyle{elsarticle-harv} 
\bibliography{els}



\end{document}